\documentclass[10pt,letter]{article}

\usepackage{amssymb}
\usepackage{amsfonts}
\usepackage{amsmath}
\usepackage{amsthm}
\usepackage{url}

\newtheorem{LEM}{Lemma} 
\newtheorem{THE}{Theorem}

\newtheorem{PRO}{Proposition} 
\newtheorem{DEF}{Definition} 
\newtheorem{example}{Example}
 
\def\hy{\hbox{-}\nobreak\hskip0pt} 

\newcommand{\SB}{\{\,}%
\newcommand{\SM}{\;{:}\;}%
\newcommand{\SE}{\,\}}%

\newcommand{\Card}[1]{|#1|}

\let\phi=\varphi
\let\epsilon=\varepsilon

\newcommand{\SSS}{\mathsf{S}}

\newcommand{\CCC}{\ensuremath{\mathcal{C}}}% \newcommand{\DDD}{\mathcal{D}}}

\newcommand{\ComplClass}[1]{\text{\normalfont #1}}

\newcommand{\NP}{\ComplClass{NP}}
\newcommand{\coNP}{co-\NP{}}

\newcommand{\stableset}{\text{\normalfont AS}}
\newcommand{\at}{\text{\normalfont at}}

\newcommand{\ta}[1]{\text{\normalfont ta($#1$)}}

\newcommand{\class}[1]{{\bf #1\normalfont}}

\newcommand{\pnot}{\neg}
\newcommand{\por}{\vee}

\usepackage{xspace}

\usepackage{microtype}

\usepackage{subfig}
\usepackage{tikz}
\usetikzlibrary{calc,arrows,shapes,petri}

\begin{document}
\title{The Good, the Bad, and the Odd: \\ Cycles in Answer-Set Programs}
\author{Johannes Klaus Fichte\\
Institute of Information Systems\\
Vienna University of Technology, Vienna, Austria\\
fichte@kr.tuwien.ac.at}
\maketitle
\begin{abstract}
Backdoors of answer-set programs are sets of atoms that represent ``clever reasoning shortcuts'' through the search space. Assignments to backdoor atoms reduce the given program to several programs that belong to a tractable target class. Previous research has considered target classes based on notions of acyclicity where various types of cycles (good and bad cycles) are excluded from graph representations of programs. We generalize the target classes by taking the parity of the number of negative edges on bad cycles into account and consider backdoors for such classes. 
We establish new hardness results and non-uniform polynomial-time tractability 
relative to directed or undirected cycles. 
% \textbf{Key Words:} Answer-Set Programming, Nonmonotonic Reasoning
\end{abstract}

\section{Introduction}
Answer-set programming (ASP) is a popular framework to describe concisely search and combinatorial problems~\cite{MarekTruszczynski99,Niemela99}.
It has been successfully applied in crypto-analysis, code optimization, the semantic web, and several other fields~\cite{Schaub08}. 
Problems are encoded by rules and constraints into disjunctive logic programs whose solutions are answer-sets (stable models). The complexity of finding an answer-set for a disjunctive logic program is $\Sigma^P_2$-complete~\cite{EiterGottlob95}. However this hardness result does not exclude quick solutions for large instances if we can exploit structural properties that might be present in real-world instances.

Recently, Fichte and Szeider~\cite{FichteSzeider11} have established a new approach to ASP based on the idea of backdoors, a concept that originates from the area of satisfiability~\cite{WilliamsGomesSelman03}.
Backdoors exploit the structure of instances by identifying sets of atoms that are important for reasoning. 
A \emph{backdoor} of a disjunctive logic program is a set of variables such that any instantiation of the variables yields a simplified logic program that lies in a class of programs where the decision problem we are interested in is tractable. By means of a backdoor of size $k$ for a disjunctive logic program we can solve the program by solving all the $2^k$ tractable programs that correspond to the truth assignments of the atoms in the backdoor.
For each answer set of each of the $2^k$ tractable programs we need to check whether it gives rise to an answer set of the given program.  In order to do this efficiently we consider tractable programs that have a small number of answer sets (e.g., stratified programs~\cite{GelfondLifschitz88}).

We consider target classes based on various notions of acyclicity on the  \emph{directed/undirected dependency graph} of the disjunctive logic program. A cycle is \emph{bad} if it contains an edge that represents an atom from a negative body of a rule. Since larger target classes facilitate smaller backdoors, we are interested in large target classes that allow small backdoors and efficient algorithms for finding the backdoors.

\subsection*{Contribution}
In this paper, we extend the backdoor approach of~\cite{FichteSzeider11} using ideas from Zhao~\cite{Zhao02}. We enlarge the target classes by taking the parity of the number of negative edges or vertices on bad cycles into account and consider backdoors with respect to such classes. This allows us to consider larger classes that also contain non-stratified programs.
Our main results are as follows:
\begin{enumerate}
	\item For target classes based on directed bad even cycles, the detection of backdoors of bounded size is intractable (Theorem~\ref{the:paranp}).
	 \item For target classes based on undirected bad even cycles, the detection of backdoors is polynomial-time tractable (Theorem~\ref{the:xp}).
\end{enumerate}
The result (2) is a \emph{non-uniform} polynomial-time result since the order of the polynomial depends on the backdoor size. An algorithm is \emph{uniform} pol\-y\-no\-mial-time tractable if it runs in time $\mathcal{O}(f(k)\cdot n^c)$ where $f$ is an arbitrary function and $c$ is a constant independent from $k$. Uniform polynomial-time tractable problems are also known as fixed-parameter tractable problems~\cite{DowneyFellows99}.
We provide strong theoretical evidence that result (2) cannot be extended to uniform polynomial-time tractability. Further, we establish that result (2) generalizes a result of Lin and Zhao~\cite{LinZhao04}.

\section{Formal Background}
We consider a universe $U$ of propositional \emph{atoms}.  A
\emph{literal} is an atom $a\in U$ or its negation $\neg a$.
A \emph{disjunctive logic program} (or simply a \emph{program}) $P$ is a
set of \emph{rules} of the following form
\begin{eqnarray}
  x_1\por \dots \por x_l \quad\leftarrow\quad y_1,\dots,y_n,\pnot z_1,\dots,\pnot z_m.
\end{eqnarray}
where $x_1,\dots,x_l, y_1,\dots,y_n, z_1,\dots, z_m$ are atoms and
$l,n,m$ are non-negative integers. Let $r$ be a rule. We write $\{x_1,\dots,x_l\}=H(r)$
(the \emph{head} of $r$) and $\{y_1,\dots,y_n,z_1,\dots,z_m\}=B(r)$ (the
\emph{body} of $r$). We abbreviate the positive literals of the body by $B^+(r)= \{y_1,\dots,y_n\}$ and the negative literals by $B^-(r)=
\{z_1,\dots,z_m\}$.  We denote the sets of atoms occurring in a rule $r$
or in a program $P$ by $\at(r)=H(r) \cup B(r)$ and $\at(P)=\bigcup_{r\in
  P} \at(r)$, respectively. A rule $r$ is \emph{normal} if $\Card{H(r)}=1$. A rule is \emph{Horn} if normal and $\Card{B^-(r)}=0$. We say that a program has a certain property if all its rules have the property. \class{Horn} refers to the class of all Horn programs.

A set $M$ of atoms \emph{satisfies} a rule $r$ if $(H(r)\,\cup\, B^-(r))
\,\cap\, M \neq \emptyset$ or $B^+(r) \setminus M \neq \emptyset$.  $M$ is a
\emph{model} of $P$ if it satisfies all rules of $P$.  The \emph{Gelfond-Lifschitz (GL)
  reduct} of a program $P$ under a set $M$ of atoms is the program $P^M$
obtained from $P$ by first removing all rules $r$ with $B^-(r)\cap M\neq
\emptyset$ and second removing all $\neg z$ where $z \in B^-(r)$ from the  remaining rules $r$
\cite{GelfondLifschitz91}.  $M$ is an \emph{answer-set} (or \emph{stable
  set}) of a program $P$ if $M$ is a minimal model of $P^M$. We
denote by $\stableset(P)$ the set of all answer-sets of~$P$. 
The main computational problems in ASP are:
\begin{itemize}
	\item \textsc{Consistency}: given a program $P$, does $P$ have an answer-set? 
	\item \textsc{Credulous/Skeptical Reasoning}: given a program $P$ and an atom $a\in \at(P)$, is $a$ contained in some/all answer-set(s) of $P$? 
	\item \textsc{AS Counting}: how many answer-sets does $P$ have? 
	\item \textsc{AS Enumeration}: list all answer-sets of $P$.
\end{itemize}

A \emph{truth assignment} is a mapping $\tau:X\rightarrow
\{0,1\}$ defined for a set $X\subseteq U$ of atoms. For $x\in X$ we put 
$\tau(\pnot x)=1 - \tau(x)$. By $\ta{X}$ we denote
the set of all truth assignments $\tau:X\rightarrow \{0,1\}$.
Let $\tau\in \ta{X}$ and $P$ be a program.  

\subsection{Strong Backdoors}
Backdoors are small sets of atoms which can be used to simplify the considered computational problems in ASP. They have originally been introduced by Williams, Gomes, and Selman~\cite{WilliamsGomesSelman03,WilliamsGomesSelman03a} as a concept to the analysis of decision heuristics in propositional satisfiability~\cite{GaspersSzeider11}. Fichte and Szeider~\cite{FichteSzeider11} have recently adapted backdoors to the field of ASP. First, we define a reduct of a program with respect to a given set of atoms. Subsequently, we give the notion of strong backdoors. In the following we refer to $\CCC$ as the \emph{target class} of the backdoor.

\begin{DEF}\label{def:truthassignment-reduct}
Let $P$ be a program, $X$ a set of atoms, and $\tau\in \ta{X}$. The \emph{truth assignment reduct}
of $P$ under $\tau$ is the logic program $P_\tau$ obtained %from $P$ 
by
\begin{enumerate}
\item removing all rules $r$ with $H(r)\cap \tau^{-1}(1)\neq \emptyset$
  or $H(r)\subseteq X$;
\item removing all rules $r$ with $B^+(r) \cap \tau^{-1}(0)\neq
  \emptyset$;
\item removing all rules $r$ with $B^-(r) \cap \tau^{-1}(1)\neq
  \emptyset$;
\item removing from the heads and bodies of the remaining rules all
  literals $v,\pnot v$ with $v\in X$.
\end{enumerate}
\end{DEF}
% The \emph{truth assignment reduct}~\cite{FichteSzeider11} of $P$ under $\tau$ is the logic program $P_\tau$ obtained from $P$ by: removing all rules $r$ with $H(r)\cap \tau^{-1}(1)\neq \emptyset$ or $H(r)\subseteq X$; removing all rules $r$ with $B^+(r) \cap \tau^{-1}(0)\neq\emptyset$; removing all rules $r$ with $B^-(r) \cap\tau^{-1}(1)\neq\emptyset$; removing from the heads and bodies of the remaining rules all  literals $v,\pnot v$ with $v\in X$.
%
%
%

\begin{DEF}%[Strong $\CCC$\hy backdoor]
	%Let $P$ be a program and $X$ a set of atoms. 
A set $X$ of atoms is a \emph{strong $\CCC$\hy backdoor} of a program $P$ if $P_{\tau}\in \CCC$ for all truth assignments $\tau\in \ta{X}$.
We define the problem of finding strong backdoors as follows: 
$k$-\textsc{Strong $\CCC\hy$Backdoor Detection}: given a program $P$, find a strong $\CCC$\hy backdoor $X$ of $P$ of size at most $k$, or report that such $X$ does not exist.
\end{DEF}

\begin{example}\label{ex:strong-bds}
Consider the program 
% \begin{eqnarray*}
\(
	P=\{ b \leftarrow a;\, d \leftarrow a;\, b \leftarrow \pnot c;\, a \leftarrow d, \pnot c;\, a \por c \leftarrow d, \pnot b;\, d \}.
\)
% \end{eqnarray*}
The set $X=\{b,c\}$ is a strong $\class{Horn}$\hy backdoor since the truth assignment reducts
\(
	P_{b=0,c=0}=P_{00}=\{ \leftarrow a;\, d \leftarrow a;\, a \leftarrow d;\, d \}		
\),
\(
	P_{01}=\{ \leftarrow a;\, d \leftarrow a;\, d \}
\),
\(
	P_{10}=\{ d \leftarrow a;\, a \leftarrow d;\, d \}
\), and
\(
	P_{11}=\{ d \leftarrow a;\, d \}
\) 
are in the target class \class{Horn}.
\end{example}

\begin{DEF}
Let $P$ be a program and $X$ a set of atoms. We define 
% \begin{eqnarray}
\[
\stableset(P,X) = \SB M\cup \tau^{-1}(1) \SM \tau\in \ta{X\cap\, \at(P)}, M \in \stableset(P_\tau)\SE\,.
\]
% \end{eqnarray}
\end{DEF} 

\begin{LEM}[\cite{FichteSzeider11}]%[ASP Backdoors]
	\label{lem:subset}
  $\stableset(P) \subseteq \stableset(P,X)$ holds for every program $P$
  and every set $X$ of atoms.
\end{LEM}

Figure~\ref{fig:exploit-backdoors} illustrates how we can exploit backdoors to find answer sets of a program. Once we have found a strong $\CCC$-backdoor $X$, we can simplify the program $P$ to programs which belong to the target class $\mathcal{C}$. Then we consider all $\Card{\ta{X}}$ truth assignments to the atoms in the backdoor $X$. We compute the answer sets $\stableset(P_\tau)$ for all $\tau\in \ta{X}$. Finally, we obtain the answer set $\stableset(P)$ by checking for each $M\in \stableset(P_\tau)$ whether it gives rise to an answer-set of $P$.

 % The answer-sets  of $P_\tau$ are $\stableset(P_{a})=\emptyset$ and  $\stableset(P_{\bar a})=\{\{b,c,d\}\}$. Thus $\stableset(P,X)=\{ \{b,c,d\}\}$, and since $\{b,c,d\}$ is an answer-set of $P$, we obtain $\stableset(P)=\{\{b,c,d\}\}$. 

% Program P with various cycles, graph and truth assignment reduct
\begin{figure}
	\begin{tikzpicture}[-latex,node distance=3em]
		\node(F)[align=center]{Find $\CCC$-backdoor\\ \(X \subseteq \at(P)\)};
		\node(P)[below of=F,node distance=7em]{\(P\)};
		\node(RUNTM1)[below of=F,node distance=14em]{?};
		\node(A)[align=center,right of=F,node distance=8em]{Apply\\ \(\tau_i: X \rightarrow\{0,1\}\) };%\in 2^B\)
		\node(Pt1)[below of=A]{\(P_{\tau_1}\in \CCC\)};
		\node(Pt2)[below of=Pt1]{\(P_{\tau_2}\in \CCC\)};
		\node(Ptn)[below of=Pt2]{\(\cdots\)};
		\node(Pt2x)[below of=Ptn]{\(P_{\tau_{\Card{\ta{X}}}}\in \CCC\)};
		\path(P) edge node[above] {\(\tau_1\)} (Pt1);
		\path(P) edge node[above] {\(\tau_2\)} (Pt2);
		\path(P) edge node[above] {\ldots} (Ptn);
		\path(P) edge node[below] {\(\tau_{\Card{\ta{X}}}\)} (Pt2x);			
		\node(RUNTM2)[below of=A,node distance=14em]{\(\mathcal{O}(|x| \cdot {\Card{\ta{X}}})\)};
		\node(C)[align=center,right of=A,node distance=8em]{Determine answer\\ sets of candidates};
		\node(ASPt1)[below of=C]{\(\stableset(P_{\tau_1})\)};
		\node(ASPt2)[below of=ASPt1]{\(\stableset(P_{\tau_2})\)};
		\node(ASPtn)[below of=ASPt2]{\(\cdots\)};
		\node(ASPt2x)[below of=ASPtn]{\(\stableset(P_{\tau_{\Card{\ta{X}}}})\)};
		\path(Pt1) edge node[above] {} (ASPt1);
		\path(Pt2) edge node[above] {} (ASPt2);
		\path(Ptn) edge node[above] {} (ASPtn);
		\path(Pt2x) edge node[below] {} (ASPt2x);
		\node(RUNTM3)[below of=C,node distance=14em]{\(\mathcal{O}(|x|^c\cdot{\Card{\ta{X}}})\)};
		\node(Cand)[align=center,right of=C,node distance=9em]{Check\\ candidates};
		\node(AS)[below of=Cand, node distance=7em]{$\stableset(P,X)$};
		\path(ASPt1) edge node[above right] {\(\cup \tau^{-1}_1(1)\)} (AS);
		\path(ASPt2) edge node[above] {\ldots} (AS);
		\path(ASPtn) edge node[above] {\ldots} (AS);
		\path(ASPt2x) edge node[below right] {\(\cup \tau^{-1}_{\Card{\ta{X}}}(1)\)} (AS);
		\node(RUNTM4)[below of=Cand,node distance=14em]{\(\mathcal{O}(\sum_{i=1}^{\Card{\ta{X}}} AS(P_{\tau_i})) \)};
		\node(SOL)[right of=Cand,node distance=6em]{Solutions};
		\node(ASP)[below of=SOL,node distance=7em]{$\stableset(P)$};
		\path(AS) edge (ASP);
		\node(RUNTM5)[below of=SOL,node distance=14em]{};
	\end{tikzpicture}
	\caption{Exploit pattern of ASP backdoors if the target class \(\mathcal{C}\) is enumerable.}
	\label{fig:exploit-backdoors}
	% \vspace{-0.6cm}
\end{figure}
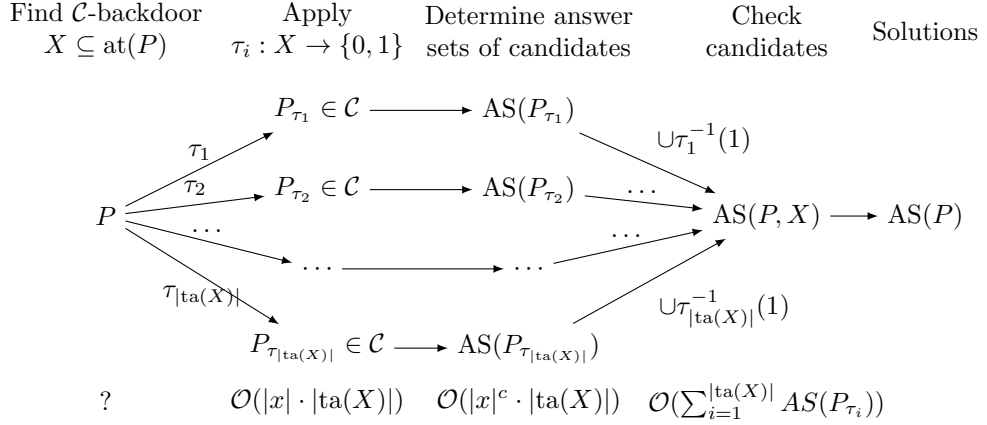
\begin{example}
	We consider the program of Example~\ref{ex:strong-bds}. The answer-sets of $P_\tau$ are $\stableset(P_{00})=\{\{a,d\}\}$, $\stableset(P_{01})=\{\{d\}\}$, $\stableset(P_{10})=\{\{a,d\}\}$, and  $\stableset(P_{11})=\{\{d\}\}$. Thus $\stableset(P,X)=\{ \{a,d\},\{c,d\},\{a,b,d\},\{b,c,d\}\}$, and since $\{c,d\}$ and $\{a,b,d\}$ are answer-sets of $P$, we obtain \(\stableset(P)=\{\{a,b,d\},\{c,d\}\}\). 
\end{example}

\begin{DEF}
  A class $\CCC$ of programs is \emph{enumerable} if for each $P\in
  \CCC$ we can compute $\stableset(P)$ in polynomial time.
\end{DEF}

% TODO: tautological rules
% For any class $\CCC$ of programs we denote by $\CCC^*$ the class
% containing all programs that belong to $\CCC$ after removal of
% tautological rules and constraints. It is easy to see that whenever
% $\CCC$ is enumerable, then so is $\CCC^*$. Note that all classes considered in this paper are enumerable.

% \begin{PRO}[\cite{FichteSzeider11}]\label{the:evaluation}
%   Let $\CCC$ be a enumerable class of programs.  Problems
%   \textsc{Consistency}, \textsc{Credulous} and \textsc{Skeptical
%     Reasoning}, \textsc{AS Counting} and \textsc{AS Enumeration} are all
%   polynomial-time solvable, assuming that the backdoor is given as an
%   input.
% \end{PRO}
% \begin{proof}
%   Let $X$ be the given backdoor.  Since we have
%   $\Card{\stableset(P,X)}\leq 2^{\Card{X}}$, we can solve each listed
%   problem by making at most $2^{\Card{X}}$ polynomial checks. 
% \end{proof}
% 
% \remark{}
% TODO: disjunctive programs, TODO: fixed-parameter tractable

\subsection{Deletion Backdoors}
For a program~$P$ and a set $X$ of atoms we define $P-X$ as the program obtained from $P$ by deleting all atoms contained in $X$ from the heads and bodies of all the rules of~$P$ and their negations. The definition gives rise to deletion backdoors and the problem of finding deletion backdoors, which is in some cases easier to solve than the problem of finding strong backdoors.

\begin{DEF}[Deletion $\CCC$\hy backdoor]
Let $\CCC$ be a class of programs. A set $X$ of atoms is a \emph{deletion $\CCC$\hy backdoor} of a program $P$ if $P-X\in \CCC$. 
We define the problem $k$-\textsc{Deletion $\CCC$-Backdoor Detection} as follows: given a program $P$, find a deletion $\CCC$\hy backdoor $X$ of $P$ of size at most $k$, or report that such $X$ does not exist. 
\end{DEF}

\subsection{Target Classes}
As explained above, we need to consider target classes of programs that only have a small number of answer sets. There are two causes for a program to have a large number of answer sets: (i)~disjunctions in the heads of rules, and (ii)~certain cyclic dependencies between rules. Disallowing both causes yields so-called \emph{stratified} programs~\cite{GelfondLifschitz88}. In the following we require normality and consider various types of acyclicity to describe target classes.
In order to define acyclicity we associate with each normal program $P$ its \emph{directed dependency graph} $D_P$ \cite{AptBlairWalker88}, 
and its \emph{undirected dependency graph} $U_P$ \cite{GottlobScarcelloSideri02}. $D_P$ has as vertices the atoms of $P$ and a directed edge $(x,y)$ between any two atoms $x,y$ for which there is a rule $r\in P$ with $x\in H(r)$ and $y\in B(r)$; if there is a rule $r\in P$ with $x\in H(r)$ and $y\in B^-(r)$, then the edge $(x,y)$ is called a \emph{negative edge}. $U_P$ is obtained from $D_p$ by replacing each negative edge $e=(x,y)$ with two undirected edges $\{x,v_e\},\{v_e,y\}$ where $v_e$ is a new \emph{negative vertex}, and by replacing each remaining directed edge $(u,v)$ with an undirected edge $\{u,v\}$. By an \emph{(un)directed cycle of  $P$} we mean an (un)directed cycle in $D_P$ ($U_P$). An (un)directed cycle is \emph{bad} if it contains a negative edge (a negative vertex), otherwise it is \emph{good}.
\newcommand{\DBC}[0]{\class{no\hy DBC}\xspace}
\newcommand{\BC}{\class{no\hy BC}\xspace}
\newcommand{\DC}{\class{no\hy DC}\xspace}
\newcommand{\C}{\class{no\hy C}\xspace}

In recent research, Fichte and Szeider~\cite{FichteSzeider11} have considered target classes that consist of normal programs without directed bad cycles (\DBC), without undirected bad cycles (\BC), without directed cycles (\DC), and without undirected cycles (\C). \DBC is exactly the class that contains all stratified programs~\cite{AptBlairWalker88}. Fichte and Szeider have examined the problems $k$-\textsc{Strong $\CCC$\hy Backdoor Detection} and $k$-\textsc{Deletion $\CCC$\hy Backdoor Detection} on the target classes $\CCC\in\{\C,\BC,\DC,\DBC\}$. 

\newcommand{\DBEC}{\class{no\hy DBEC}\xspace}
\newcommand{\BEC}{\class{no\hy BEC}\xspace}
\newcommand{\DEC}{\class{no\hy DEC}\xspace}
\newcommand{\EC}{\class{no\hy EC}\xspace}

% TODO: Style
\begin{example}
The set $X=\{a,b\}$ is a deletion $\DBEC$\hy backdoor of the program $P$ of Example~\ref{ex:strong-bds}, since the simplification
\(
P-X =\{ d;\, \leftarrow \pnot c;\, \leftarrow d, \pnot c\}
\)
is in the target class \DBEC. We observe easily that there exists
no deletion \DBEC-backdoor of size $1$.
% Thus there is a strong $\class{Horn}$\hy backdoor that is smaller than each deletion $\class{Horn}$\hy backdoor of $P$.

\end{example}
% \remark{} In general, not every strong $\CCC$\hy backdoor is a deletion $\CCC$\hy backdoor, and not every deletion $\CCC$\hy backdoor is a strong $\CCC$\hy backdoor.

\section{Parity Cycles}
In this section, we generalize the acyclicity based target classes by taking the parity of the number of negative edges (vertices) into account and consider backdoors for such classes. We say that an (un)directed cycle in a given program $P$ is \emph{even} if the cycle has an even number of negative edges (vertices). 
The definition gives rise to the new target classes of all \emph{normal} programs without directed bad even cycles (\DBEC), without undirected bad even cycles (\BEC), without directed even cycles \text{(\DEC)}, and without even cycles (\EC).

\begin{example}
	For instance in the program $P$ of Example~\ref{ex:strong-bds} the sequence $(a,b,c,a)$ is a directed bad even cycle, $(a,b,v_{(b,c)},c,v_{(c,a)},a)$ is an undirected bad even cycle, $(a,d,a)$ is a directed even cycle, and $(a,b,v_{(b,c)},c,v_{(c,a)},a)$ is an undirected even cycle (see Figure~\ref{fig:DpUp}). The set $X=\{c\}$ is a strong $\DBEC$\hy backdoor since the truth assignment reducts $P_{c=0}=P_0=\{ b \leftarrow a;\, d \leftarrow a;\, b;\, a \leftarrow d;\, a \leftarrow d, \pnot b;\, d \}$ and $P_{1}=\{ b \leftarrow a;\, d \leftarrow a;\, d \}$ are in the target class \DBEC. The answer-sets  of $P_\tau$ are $\stableset(P_{0})=\{\{a,b,d\}\}$ and  $\stableset(P_{1})=\{\{d\}\}$. Thus $\stableset(P,X)=\{ \{a,b,d\},\{c,d\}\}$, and since $\{a,b,d\}$ and $\{c,d\}$ are answer-sets of $P$, we obtain $\stableset(P)=\{\{a,b,d\},\{c,d\}\}$. 
\end{example}
%Figure~\ref{fig:ex1} illustrated the dependency graphs of $P$. 

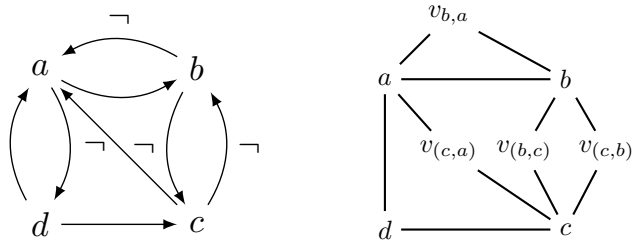
\begin{figure}[htb]
	\centering
	\subfloat{
	\scalebox{1.3}{
		\begin{tikzpicture}[-latex,node distance=16mm]
			\node(a){$a$};
			\node(b)[right of=a]{$b$};
			\node(c)[below of=b]{$c$};
			\node(d)[left of=c]{$d$};
			\path(a) edge[bend right] (b);
			\path(c) edge node[transparent,label=left:$\pnot$]{} (a); %[bend right]
			\path(b) edge[bend right] node[transparent,label=left:$\pnot$]{} (c);
			\path(d) edge (c);
			\path(d) edge[bend left] (a);
			\path (a) edge[bend left] (d);
			\path (b) edge[bend right] node[transparent,label=above:$\pnot$]{} (a);
			\path (c) edge[bend right] node[transparent,label=right:$\pnot$]{} (b);
		\end{tikzpicture}}
		\label{fig:Dp}
	}\qquad\quad
	\subfloat{
		\begin{tikzpicture}[every edge/.style={draw,thick},-,node distance=12mm]
			\node(a){$a$};
			\node(b)[right of=a,xshift=+12mm]{$b$};
			\node(vca)[below right of=a,yshift=-1mm]{$v_{(c,a)}$};
			\node(vbc)[below left of=b,xshift=3mm,yshift=-1mm]{$v_{(b,c)}$};
			\node(vcb)[below right of=b,xshift=-3mm,yshift=-1mm]{$v_{(c,b)}$};
			\node(d)[below of=a,yshift=-8mm]{$d$};
			\node(c)[right of=d,xshift=12mm]{$c$};
			\node(vba)[above right of=a]{$v_{b,a}$};
			
			\path(a) edge (vba);
			\path(vba) edge (b);
			\path(a) edge (b);
			\path(c) edge (vca);
			\path(vca) edge (a); %[bend right]
			\path(b) edge (vbc);
			\path(vbc) edge (c);
			\path(d) edge (c);
			\path(d) edge (a);
			\path(b) edge (vcb);
			\path(vcb) edge (c);
		\end{tikzpicture}
		\label{fig:Up}
	}
	\caption{Directed dependency graph $D_P$ (left) and undirected dependency graph $U_P$ (right) of the program $P$ of Example~\ref{ex:strong-bds}.}
	\label{fig:DpUp}
	% \vspace{-0.6cm}
\end{figure}

\subsection{Computing Answer-Sets}
First, we discuss the connection between the problem of finding bad even cycles in signed graphs and even cycles in graphs. 
A \emph{signed (directed) graph} is a graph whose edges are either positive (unlabeled) or negative. 
We construct the \emph{unlabeled directed graph} $G'$ of a signed directed graph $G=(V,E)$ as follows: we replace in $G$ each positive edge $e=(u,v) \in E$ by two edges $(u,v_e),\,(v_e,v)$ where $v_e$ is a new vertex. Then we remove the labels from the negative edges. Analogously, we construct the \emph{unlabeled undirected graph} where we ignore the direction of the edges.   
The following connection was already observed by Aracena, Gajardo, and Montalva~\cite{MontalvaAracenaGajardo08}. 

\begin{LEM}[\cite{MontalvaAracenaGajardo08}]\label{lem:badeven2even}
	A signed (un)directed graph $G$ has an even cycle if and only if its unlabeled (un)directed graph $G'$ has a cycle of even length.
\end{LEM}

\begin{proof}
Let $G=(V,E)$ be the signed directed graph and $G'=(V',E')$ its unlabeled directed graph. Since every positive edge $e \in E$ corresponds to two edges $e_1,e_2 \in E'$ % positive edges give always even many edges in $G'$. Because 
and every negative edge $e \in E$ corresponds to one edge $e \in E'$, a cycle in $G$ with an even number of negative edges gives a cycle of even length in $G'$.
%and a cycle with an odd number of negative edges gives an odd cycle in $G'$.
%
Conversely, let $G'=(V',E')$ be an unlabeled directed graph that contains a cycle of even length. Then $G$ contains an even cycle since every two edges $e_1,e_2 \in E'$ correspond either to two negative edges or no negative edge. 
The proof works analogously for undirected graphs. 

% in $G$ or one positive edge in $G$, an even cycle in $G'$ corresponds to a cycle with an even number of negative edges.
\end{proof}

\newcommand{\WFM}{\ensuremath{\text{WFM}}}

The well-founded reduct of a program $P$ under an interpretation $\tau$ is the logic program $P^{\text{WF}}_\tau$ we obtain by
removing all rules $r \in P$ where some $\at(r) \in \tau^{-1}(0)$, and
removing from all rules $r \in P$ all literals $x \in \at(r)$ where $x \in \tau^{-1}(1)$.
We obtain the program $P^+$ ($P^-$) by removing all rules from $P$ where $B^+(r)\neq \emptyset$ ($B^-(r)\neq \emptyset$ respectively). The \emph{well-founded model} $\WFM(P)$ of a program $P$ is the least fixed point of the sequence of interpretations where $\tau_0:=\emptyset$, $\tau_{k+1}^{-1}(1)$ consists of the least model of $P^+(P^\text{WF}_\tau)$ and $\tau_{k+1}^{-1}(0)$ consists of the atoms of $P^+(P^\text{WF}_\tau)$ that are not in the least model.
% \begin{LEM}[\cite{Zhao02}]
% For each target class $\CCC \in \{$\class{DBEC\hy Acyc}, \class{BEC\hy Acyc,} \class{DEC,} \class{EC\hy Acyc}$\}$ we can compute the set $\stableset(P)$ of answer-sets of a program $P \in \CCC$ in polynomial time and $\Card{\stableset(P)} \leq 1$.% is a singleton or the emptyset.
% \end{LEM}
% \begin{proof}[Sketch]
% Zhao~\cite{Zhao02} has shown that a program without bad even cycles has either no answer-set or the well-founded model is its answer-set. Thus models can be computed in polynomial time~\cite{Gelder89,GelderRossSchlipf91}. Let $P$ be a program and $D_P$ ($U_P$) its (un)directed dependency graph. Since every bad even cycle in $D_P$ is also a bad even cycle in $U_P$, this holds for the undirected case. 
% % The unlabeled cases are true because of Lemma~\ref{lem:badeven2even}.
% %The well-founded model of a normal program can be
% \end{proof}

\begin{LEM}\label{lem:enumerable}
The target classes $\DBEC, \BEC,\DEC, \EC$ are enumerable. 
\end{LEM}
\begin{proof}
	
	Zhao~\cite{Zhao02} has shown that a program without a bad even cycle has either no answer-set or the well-founded model is its answer-set. Since in the definition of the well-founded model the sequence of $\tau_0,\tau_1,\ldots$ is monotone for a normal program, there is a least fixed point and it can be computed in polynomial time~\cite{Gelder89,GelderRossSchlipf91}. Thus the answer sets can be computed in polynomial time. Let $P$ be a program and $D_P$ ($U_P$) its (un)directed dependency graph. Since every bad even cycle in $D_P$ is also a bad even cycle in $U_P$, this holds for the undirected case. Considering the fact that every bad even cycle in $D_P$ is also an even cycle in $D_P$, the lemma sustains for the target class $\DEC$. Since every bad even cycle in $D_P$ is also an even cycle $U_P$, it prevails for the remaining target class $\EC$.
	%The only case where a directed even cycle is not an undirected even cycle are unlabeled digones in $D_P$, e.g. $P'=\{x\leftarrow y,\ y\leftarrow x\}$. Since unlabeled digones are not  

\end{proof}

\begin{PRO}
\label{cor:evaluation}
The problems  \textsc{Consistency}, \textsc{Credulous} and \textsc{Skeptical Reasoning}, \textsc{AS Counting} and \textsc{AS Enumeration} are all polynomial-time solvable for programs with strong $\CCC$\hy backdoor of bounded size,  $\CCC \in \{\DBEC,\\ \BEC,\DEC, \EC\}$, assuming that the backdoor is given as an input.
\end{PRO}
\begin{proof}
	
	Let $X$ be the given backdoor. By Lemma~\ref{lem:enumerable} each target class $\CCC$ is enumerable. Since we have $\Card{\stableset(P,X)}\leq 2^{\Card{X}}$, we can solve each listed problem by making at most $2^{\Card{X}}$ polynomial checks.
%Since the Lemma ... gives an upper bound for the number of answer sets
%this works for disjunctive programs too
%\marginpar{TODO: check: disjunctive vs. normal programs}
	
\end{proof}
% Since it can be recognized in polynomial time if a program belongs to a the target class, the considered problems are obviously solvable in non-uniform polynomial time for fixed backdoors. 
If the problem of determining backdoors is also polynomial-time solvable with respect to the fixed size of a smallest strong $\CCC$\hy backdoor, then the ASP problems are polynomial-time solvable.
%TODO: check this here
% Considering disjunctive logic programs, we can not ensure that the above mentioned problems are polynomial-time solvable when the size of the backdoor is given as an input, since each check works in co-\NP.

\begin{LEM}
	
	For all target classes $\CCC\in \{\DBEC, \BEC,\DEC, \EC\}$ every deletion $\CCC$\hy backdoor is also a strong $\CCC$\hy backdoor.
	
\end{LEM}
\begin{proof}
	We show the statement by proving that $P_\tau \subseteq P-X$ for every $\tau\in \ta{X}$ and for every program $P\in \CCC$. Let $P$ be a program and $X$ a set of atoms of $P$. We choose arbitrarily a truth assignment $\tau \in \ta{X}$. For a rule $r\in P$ if $H(r)\cap \tau^{-1}(1)\neq \emptyset$ or $H(r)\subseteq X$ or $B^+(r) \cap \tau^{-1}(0)\neq \emptyset$ or $B^-(r) \cap \tau^{-1}(1)\neq \emptyset$, then $r$ is removed from $P$ by the truth assignment reduct of $P$ under $\tau$. 
	However removing all literals $x, \pnot x$ with $x \in X$ from the head $H(r)$ and the body $B(r)$ yields a new rule $r' \in P-X$. Thus $r' \notin P_\tau$ and $r' \in P-X$ where $r' \subseteq r$. If the conditions (1), (2), and (3) of Definition~\ref{def:truthassignment-reduct} above do not apply, then all literals $v,\pnot v$ with $v\in X$ are removed from the heads $H(r)$ and bodies $B(r)$ by the truth assignment reduct of $P$ under $\tau$. This is also done by $P-X$.  Hence $P_\tau \subseteq P-X$.
	
\end{proof}

\subsection{Backdoor Detection for Directed Target Classes}
In order to apply backdoors we need to find them first. In this section we consider the problems $k$-\textsc{Strong $\CCC$\hy Backdoor Detection} and $k$-\textsc{Deletion $\CCC$\hy Backdoor Detection} for the target classes $\CCC\in \{\DEC,\DBEC\}$.

For an unlabeled directed graph $G=(V,E)$ and fixed vertices $s,m,t\in V$ we define the \emph{program $P_{s,m,t}(G)$} as follows: For each edge $e=(v,w) \in E$ where $v,w \in V$ and $w\neq m$ we construct a rule $r_e$:
$v \leftarrow w$. For the edges $e'=(v',m)$ where $v'\in V$ we construct a rule $r_{e'}$: $v' \leftarrow \pnot m$. Then we add the rule $r_{s,t}$: $t \leftarrow \pnot s$.

\begin{LEM}\label{lem:paththrough}
Let $G = (V, E)$ be a directed graph and $s, m, t$ three distinct vertices of $G$. Then $G$ has a simple path from $s$ to $t$ via $m$ if and only if $P_{s,m,t}(G)\notin \DBEC$.
\end{LEM}
\begin{proof}

Let $G$ be a graph and and $p=(s, s_1, \ldots, s_k, m,t_1, \ldots,t_l,t)$ a path in $G$ where $s\neq m, m \neq t, s \neq t$. The construction $P_{s,m,t}$ gives rules $\{s \leftarrow s_1;\, s_1 \leftarrow s_2;\, \ldots s_k\leftarrow \neg m;\, m \leftarrow t_1;\, t_1 \leftarrow t_2;\, \ldots t \leftarrow t;\, t \leftarrow \neg s\} \in P_{s,m,t}(G)$. Since $D_P$ contains the cycle $c=(s,s_1,\ldots,s_k,m,t_1,\ldots,t,s)$ and $c$ contains an even number of negative edges, the program $P_{s,m,t}(G) \notin \DBEC$. 

Conversely, let $P_{s,m,t}(G) \in \DBEC$, then $P_{s,m,t}(G)$ contains a bad even cycle $c$. Since the construction of $P_{s,m,t}(G)$ gives only negative edges $(t,s)\in D_{P_{s,m,t}(G)}$ and $(v,m) \in D_{P_{s,m,t}(G)}$ where $v \in \at(P_{s,m,t}(G))$, the cycle $c$ must have the vertices $s$,$m$, and $t$. Further every rule $r_e \in P_{s,m,t}(G)$ corresponds to an edge $e \in E$. It follows that there is a simple path $s,\ldots,m,\ldots,t$.

% If $G$ has a path from $s$ to $t$ via $m$, then the new negative edge from $t$ to $s$ and the negative labeling of edges of the form $(u,m)$ yields a bad even cycle in $P_{s,m,t}(G)$. Thus $P_{s,m,t}(G) \notin \DBEC$.
% 
% Conversely, if $P_{s,m,t}(G) \notin \DBEC$, then $P_{s,m,t}(G)$ contains a bad even cycle. Since the construction of $P_{s,m,t}(G)$ gives only negative edges $(t,s)\in D_{P_{s,m,t}(G)}$ and $(v,m) \in D_{P_{s,m,t}(G)}$ where $v \in \at(P_{s,m,t}(G))$, the existing cycle must go through $s$,$m$, and $t$. It follows that there is a simple path $s,\ldots,m,\ldots,t$.
\end{proof}

\begin{THE}\label{the:paranp}
The problems $k$-\textsc{Strong} \DBEC-\textsc{Backdoor Detection} and $k$-\textsc{Deletion} \DBEC-\textsc{Backdoor Detection} are \coNP{}-hard for every constant $k\geq 0$.
\end{THE}
\begin{proof}
Let $k\geq 0$. Let $G$ be a given directed graph and $t,m,s$ vertices of $G$. It was shown by Lapaugh and Papadimitriou~\cite{LapaughPapadimitriou84} that deciding whether $G$ contains a simple path from $s$ to $t$ via $m$ is NP-complete. By Lemma \ref{lem:paththrough}, such a path exists if and only if  $P_{s,m,t}(G) \notin \DBEC$, hence recognizing \DBEC is \coNP{}-hard.
Let $G_k$ denote the graph obtained from $G$ by adding $k$ disjoint bad even cycles. Clearly $G_k$ has a deletion \DBEC-Backdoor of size $\leq k$ if and only if $P_{s,m,t}(G_k) \in \DBEC$, hence $k$-\textsc{Deletion \DBEC-Backdoor Detection} is \coNP{}-hard. Similarly, $G_k$ has a strong \DBEC-Backdoor of size $\leq k$ if and only if $P_{s,m,t}(G_k) \in \DBEC$,
and so \textsc{$k$-Strong \DBEC-Backdoor Detection} is \coNP{}-hard as well.
\end{proof}

\begin{THE}\label{the2}
Let $k>0$ be a constant. The problems $k$-\textsc{Deletion \DEC-Backdoor Detection} and $k$-\textsc{Strong \DEC-Backdoor Detection} are poly\-nomial-time trac\-ta\-ble.
\end{THE}
\begin{proof}
By Lemma~\ref{lem:badeven2even}, we can reduce to the problem of finding a cycle of even length in the unlabeled dependency graph. 
Vazirani and Yannakakis~\cite{VaziraniYannakakis88} have shown that finding a cycle of even length in a directed graph is equivalent to finding a Pfaffian orientation of a graph. Since Robertson, Seymour, and Thomas~\cite{RobertsonSeymourThomas99} have shown that a Pfaffian orientation can be found in polynomial time.
For each possible backdoor of size $k$ we need to test ${n \choose k} \leq n^k$ subsets $S\subseteq V$ of size $k$ whether $D_P - S$ contains a cycle of even length, respectively $D_{P_{\tau}}$ for $\tau \in \ta{S}$.
Since we can do this in polynomial time for each fixed $k$, the theorem follows.
\end{proof}

In Theorem~\ref{the2} we consider $k$ as a constant. In the following proposition we show that if $k$ is considered as part of the input, then the problem $k$-\textsc{Strong \DEC-Backdoor Detection} is polynomial-time equivalent to the problem \textsc{Hitting Set} and $k$ is preserved. An instance of this problem is a pair $(\SSS,k)$ where $\SSS=\{S_1,\dots,S_m\}$ is a family of sets and $k$ is an integer. The question is whether there exists a set $H$ of size at most $k$ which intersects with all the $S_i$; such $H$ is a hitting set. Note that there is strong theoretical evidence that the problem \textsc{Hitting Set} does not admit uniform poly\-nomial-time tractability~\cite{DowneyFellows99}.

\begin{PRO}\label{the:w2}
The problem $k$-\textsc{Strong \DEC-Backdoor Detection} is poly\-nomial-time equivalent to the problem \textsc{Hitting Set}.
\end{PRO}
\begin{proof} The proof is very similar to the proof for target classes without respecting the parity by Fichte and Szeider~\cite{FichteSzeider11}.
  We construct a program $P$ as follows.  As atoms we take the elements
  of $\mathcal{S}=\bigcup_{i=1}^m S_i$ and new atoms $a_i^j$ and $b_i^j$ for
  $1\leq i \leq m$, $1\leq j \leq k+1$.  For each $1\leq i \leq m$ and
  $1\leq j \leq k+1$ we take two rules $r_i^j$, $s_i^j$ where
  $H(r_i^j)=\{a_i^j\}$, $B^-(r_i^j)=S_i\cup \{b_i^j\}$, $B^+(r_i^j)=\emptyset$;
  $H(s_i^j)=\{b_i^j\}$, $B^-(s_i^j)=\{a_i^j\}$, $B^+(s_i^j)=\mathcal{S}$.

  We show that $\SSS$ has a hitting set of size at most $k$ if and only if $P$
  has a strong \DEC-backdoor of size at most $k$. Let $\SSS$ be
  a family of sets and $H$ an hitting set of $\SSS$ of size at most $k$.
%  Construct the program $P$ as stated above.
 Choose arbitrarily an atom $s_i
  \in \at(P)\cap \mathcal{S}$ and a truth assignment $\tau \in \ta{H}$. If
  $s_i \in \tau^{-1}(0)$, then $B^+(s_i^j)\cap \tau^{-1}(0) \neq \emptyset$
  for $1 \leq j \leq k+1$. Thus $s_i^j\notin P_\tau$. If $s_i\in
  \tau^{-1}(1)$, then $B^-(r_i^j) \cap \tau^{-1}(1) \neq \emptyset$ for $1
  \leq j \leq k+1$. Thus $r_i^j \notin P_\tau$. Since $H$ contains at least
  one element $e \in S$ from each set $S \in \mathcal{S}$, the truth assignment reduct
  $P_\tau \in \DEC$. We conclude that $H$ is a strong
  \DEC-backdoor of $P$ of size at most $k$.

Conversely, let $X$ be a strong \DEC-backdoor of $P$ of size at most $k$. Since the directed dependency graph $D_P$ contains $k+1$ directed even cycles $(a_i^j, b_i^j, a_i^j)$ and $a_i^j$ (respectively $b_i^j$) is contained in exactly one rule $r_i^j$ (respectively $s_i^j$), $\Card{\bigcup a_i^j}>k$ and $\Card{\bigcup b_i^j}>k$. Hence we have to select atoms from $S_i$. Since $S_i \subseteq B^-(r_i^j)$ for $1\leq i\leq m$ and $1 \leq j \leq k+1$, we have to select at least one element from each $S_i$ into the backdoor $X$. Thus we have established that $X$ is a hitting set of $\SSS$, and so the theorem follows.
\end{proof}
%$a_i^j\in r_i^j$, but $a_i^j \notin r'$ for some $r'\neq r_i^j$ and $b_i^j\in s_i^j$, but $b_i^j \notin s'$ for some $s'\neq s_i^j$ 

\subsection{Backdoor Detection for Undirected Target Classes}
The results of Theorem~\ref{the:paranp} suggest to consider the backdoor detection on the weaker target classes based on undirected even acyclicity.

\begin{LEM}
Let $P$ be a program, $P \in \EC$ can be decided in polynomial time.
\end{LEM}
\begin{proof}
Let $P$ be a program and $G$ its dependency graph $U_P$. 
Lemma~\ref{lem:badeven2even} allows to consider the problem of finding an even cycle in the unlabeled version of $U_P$. Since Yuster and Zwick~\cite{YusterZwick94} have shown that finding an even cycle in an undirected graph is polynomial-time solvable, the lemma holds.
% Since Arikati and Peled~\cite{ArikatiPeled96} have shown that finding an odd path in an undirected graph is polynomial-time solvable the lemma follows.
\end{proof}

\begin{LEM}
Let $P$ be a program. The problem of deciding whether $P \in \BEC$ can be solved in polynomial time.
\end{LEM}
% whether a given program $P$ $\CCC\hy$ \textsc{Membership} 
\begin{proof}
Let $P$ be a program and $G$ its dependency graph $U_P$. 
For a negative edge $e$ of $G$ we define $G_e$ to be the unlabeled graph of $G-e$. Now $G$ contains a bad even cycle if and only if $G$ has an edge $e=\{s,t\}$ such that $G_e$ contains an odd path from $s$ to $t$. 
Since Arikati and Peled~\cite{ArikatiPeled96} have shown that finding an odd path in an undirected graph is polynomial-time solvable, the lemma follows.
% For $\class{EC}$ we can~\cite{YusterZwick94} 
\end{proof}

\begin{THE}\label{the:xp}
Let $k>0$ be a constant. For the target classes $\CCC \in \{ \EC,\\\BEC\}$ the problems $k$-\textsc{Deletion $\CCC$\hy Backdoor Detection} and $k$-\textsc{Strong $\CCC$\hy Backdoor Detection} are non-uniform polynomial-time tractable.
\end{THE}
\begin{proof}
Let $P$ be a program and $U_P=(V,E)$ its undirected dependency graph. Let $n$ be the size of $V$. For each possible backdoor of size $k$ we need to test ${n \choose k} \leq n^k$ subsets $S\subseteq V$ of size $k$ whether $U_P - S$ contains a (bad) cycle of even length, respectively $U_{P_{\tau}}$ for $\tau \in \ta{S}$.
Since we can do this in polynomial time for each fixed $k$, the problems  $k$-\textsc{Deletion} $\CCC$\hy \textsc{Backdoor} and $k$-\textsc{Strong $\CCC$\hy Backdoor Detection} are non-uniform polynomial-time tractable.
\end{proof}

In Theorem \ref{the:xp} we consider $k$ as a constant. If $k$ is considered as part of the input we can show that for each class $\CCC\in \{\EC, \BEC\}$ the problem $k$-\textsc{Strong $\CCC$\hy Backdoor Detection} is polynomial-time equivalent to \textsc{Hitting Set}~\cite{FichteSzeider11}. As mentioned before for \DEC there is strong theoretical evidence that $k$-\textsc{Strong $\CCC$\hy Backdoor Detection} does not admit a uniform polynomial-time tractability result. 

\begin{PRO}
	
The problem $k$-\textsc{Strong \CCC\hy Backdoor Detection} is poly\-nomi\-al-time equivalent to the problem \textsc{Hitting Set} for each class $\CCC\in \{\EC,\\ \BEC\}$.
	
\end{PRO}
\begin{proof}

We modify the above reduction from \textsc{Hitting Set} by redefining the
rules $r_i^j$, $s_i^j$. We put $H(r_i^j)=\{a_i^j\}$, $B^-(r_i^j)=S_i\cup \{b_i^j\}$, $B^+(r_i^j)=S_i$; $H(s_i^j)=\{b_i^j\}$, $B^-(s_i^j)=\{a_i^j\}$,
$B^+(s_i^j)=\emptyset$.

\end{proof}

\section{Relationship between Target Classes} \label{sec:relationship}
In this section, we compare ASP parameters in terms of their \emph{generality}. We have already observed that every deletion $\CCC$\hy backdoor is a strong $\CCC$\hy backdoor for a target class $\CCC\in \{\EC,\DEC,\BEC,\DBEC\}$. For the considered target classes it is easy to see that if $\CCC \subseteq$ $\CCC'$, then every $\CCC'$ backdoor of a program $P$ is also a $\CCC$-backdoor, but there might exist smaller $\CCC'$-backdoors. Thus we compare the target classes among each other instead of the backdoors. By definition we have $\DBC \subsetneq \DBEC$, $\DEC \subsetneq \DBEC$, $\EC \subsetneq \BEC$, $\C \subsetneq \EC$, and $\DC \subsetneq \DEC$. The diagram in Fig.~\ref{fig:diagram} shows the relationship between the various classes, an arrow from $\CCC$ to $\CCC'$ indicates that $\CCC$ is a proper subset of $\CCC'$. If there is no arrow between two classes (or the arrow does not follow by transitivity of set inclusion), then the two classes are incomparable.

\begin{figure}
\centering
\begin{tikzpicture}[-latex,node distance=25mm,font=\small]
	\node[color=black](dbec){\DBEC};
	\node[left of=dbec,color=gray] (dbc) {\DBC};
	\node[left of=dbc,color=gray] (bc) {\BC};
	\node[left of=bc,color=gray] (c) {\C};
	\node[color=black,above left of=dbec] (dec) {\DEC};
	\node[left of=dec,color=gray] (dc) {\class{DC}};
	\node[color=black,below left of=dbec] (ebc) {\BEC};
	\node[color=black,left of=ebc] (ec) {\EC};
	\path(c) edge node[] (CEC) {} (ec);
	\path(c) edge (bc);
	\path(bc) edge node[] (BC2DBC) {}  (dbc);
	\path(dbc) edge node[right] (DBC2DBEC) {} (dbec);
	\path(dc) edge node[] (DC2DBC) {} (dbc);
	\path(dec) edge (dbec);
	\path(ebc) edge (dbec);
	\path(bc) edge (ebc);	
	\path(dc) edge node[above left=0.75cm] (DC2DEC) {} (dec);
	\path(ec) edge node[below =1cm] (EC2EBC) {} (ebc);
	\node[right of=ebc, transparent] (ebc2) {};
	\path(ebc) edge[transparent] node[below =1cm] (HELPER) {} (ebc2);
	\draw[-,thick] 	(HELPER) .. controls (1,-1) and (-1,-1) .. (DBC2DBEC)
					(DBC2DBEC) .. controls (-1,1) and (0,1) .. (0,2.5);
	\draw[-,thick] 	(-7,-2.8) .. controls (-7,-1) and (-6,-1) .. (BC2DBC)
					(BC2DBC) .. controls (-3,1) and (-2.9,1) .. (DC2DEC);
	\node[below of=EC2EBC,node distance=-2mm](xp){\it non-uniform};
	\node[below of=xp,node distance=4mm](xp2) {\it polynomial-time};
	\node[right of=xp,node distance=50mm]{\textit{co-NP-hard}};
	\node[left of=xp,node distance=50mm](fpt){\it uniform};
	\node[below of=fpt,node distance=4mm] (fpt2) {\it polynomial-time};
\end{tikzpicture}%
\caption{\small Relationship between classes of programs and state of the knowledge regarding the complexity of the problem \sc Deletion $\mathcal{C}\hy$Backdoor\normalfont. %The recent results are colored in gray and
The new results are colored in black.} \label{fig:diagram}
\end{figure}
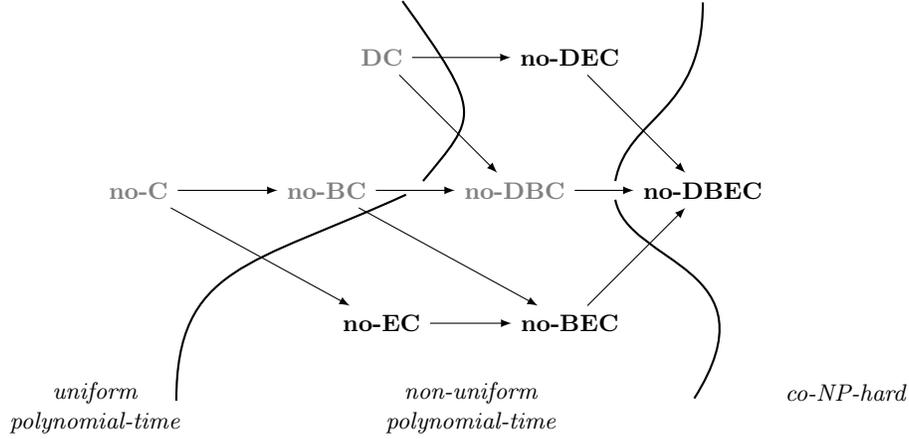

Lin and Zhao~\cite{LinZhao04} have studied even cycles as a parameter to ASP. They have proved that for fixed $k$ the main reasoning problems are polynomial-time solvable if the number of the shortest even cycles is bounded. The following proposition states that size of \DBEC-backdoors is a more general parameter than the number of even cycles.

\begin{PRO}
	There is a function $f$ such that $k \leq f(l)$ and no function $g$ such that $l < g(k)$ for all programs $P$ where $k$ is the size of the smallest deletion-$\DBEC$\hy backdoor of $P$ and $l$ is the number of even cycles in $D_P$. 
\end{PRO}
\begin{proof}
	Let $P$ be some program. If $P$ has at most $k$ bad even cycles, we can construct a \DBEC-backdoor $X$ for $P$ by taking one element from each bad even cycle into $X$. Thus there is a function $f$ such that $k \leq f(l)$. If a program $P$ has a \DBEC-backdoor of size $1$, it can have arbitrary many even cycles that run through the atom in the backdoor. It follows that there is no function $g$ such that $l < g(k)$ and the proposition holds.
\end{proof}

\section{Conclusion}

We have extended the backdoor approach of~\cite{FichteSzeider11} by taking the parity of the number of negative edges on bad cycles into account. In particular, this allowed us to consider target classes that contain non-stratified programs. 
We have established new hardness results and non-uniform polynomial-time tractability depending on whether we consider directed or undirected even cycles.
We have shown that the backdoor approach with parity target classes generalize a result of Lin and Zhao~\cite{LinZhao04}.
Since Theorem~\ref{the:paranp} states that target classes based on directed even cycles are intractable, we think these target classes are of limited practical interest.
The results of this paper give rise to research questions that are of theoretical interest. For instance, it would be stimulating to find out whether the problem $k$-\textsc{Strong $\CCC$\hy Backdoor Detection} is uniform polynomial-time solvable (fixed-parameter tractable) for the classes \BC and \BEC, which is related to the problems parity feedback vertex set and parity subset feedback vertex set.

\section*{Acknowledgement}
The author would like to thank Stefan Szeider for suggesting the new target class that takes the parity of the number of negative edges into account, for many helpful comments, and his valuable advice.
Serge Gaspers for pointing out the result of Arikati and Peled on finding odd paths in polynomial time and many helpful comments and to the anonymous referees for their helpful comments. 
The author was supported by the European Research Council (ERC), Grant COMPLEX REASON 239962.

\end{document}